%% file: main.tex
\documentclass[11pt,letterpaper]{article}
\usepackage{hyperref}
\hypersetup{
     colorlinks   = true,
     linkcolor = black,
     urlcolor    = teal,
	 citecolor = teal
}
\usepackage{authblk}
\usepackage{natbib}
\usepackage[margin=1in]{geometry}
\usepackage[utf8]{inputenc} 
\usepackage[T1]{fontenc}    
\usepackage{hyperref}       
\usepackage{url}            
\usepackage{booktabs}       
\usepackage{amsfonts}       
\usepackage{nicefrac}       
\usepackage{microtype}      
\usepackage{xcolor}

\usepackage{amsmath,amsfonts,amssymb}
\usepackage{amsthm}
\usepackage{bm}
\usepackage{enumitem,array,booktabs}
\usepackage{algorithm}
\usepackage{algorithmic}
\newcounter{procedure}

\usepackage[capitalize]{cleveref}

\usepackage{bbm}
\usepackage{tabulary,multirow}
\usepackage{makecell}
\usepackage{dsfont}
\usepackage{tcolorbox}

\usepackage{thmtools}
\usepackage{xpatch}
\xpatchcmd{\proof}{\itshape}{\scshape}{}{}
\usepackage{tablefootnote}

\newtheorem{definition}{Definition}
\newtheorem{lemma}{Lemma}

\newtheorem{proposition}{Proposition}

\newtheorem{remark}{Remark}

\input{macros}

\begin{document}

\title{Online Set Learning from Precision and Recall Feedback}
\date{\today}
\author[1]{Lee Cohen}
\author[2]{Yishay Mansour}
\author[3]{Shay Moran}
\author[4]{Han Shao}
\affil[1]{Stanford University}
\affil[2]{Tel Aviv University and Google Research}
\affil[3]{Technion and Google Research}
\affil[4]{University of Maryland}

\maketitle

\begin{abstract}

We consider the problem of learning an unknown subset $\Nt$ of a domain in an online setting. 
In each round $t$, the learner predicts a set of items ${N}_t$ and receives one of two types of feedback, each with equal probability: {\it precision feedback}, in which a randomly chosen item from the predicted set $N_t$ is revealed and the learner is told whether it belongs to $\Nt$ (incurring a reward if it does), or {\it recall feedback}, in which a randomly chosen item from the target set $\Nt$ is revealed and the learner is told whether it belongs to $N_t$ (incurring a reward if it does). The goal is to maximize the cumulative reward over time.
This simple online set learning problem abstracts a variety of learning scenarios with precision- and recall--type feedback.

We show that a hypothesis class (a family of subsets of the domain) is learnable in this setting if and only if it has finite Vapnik-Chervonenkis (VC) dimension, mirroring the classical PAC characterization. However, the resulting algorithmic structure is markedly more intricate: in contrast to standard Probably Approximately Correct (PAC) learning---where the algorithmic landscape is governed by the simple principle of Empirical Risk Minimization (ERM)---our partial feedback model can invalidate ERM and even all proper learning rules. We develop algorithms to address the dependencies induced by the feedback, obtaining regret guarantees in both the realizable and agnostic settings. Our results provide a qualitative characterization of learnability in this model, addressing its most basic question, while pointing to a range of natural and intriguing open questions, including the determination of optimal regret rates.

  \end{abstract}
  
  \section{Introduction}

  In this work, we study the problem of learning an unknown subset $\Nt$ of a domain $\cX$. A basic question in modeling such problems is how to evaluate the quality of a predicted set~$\hat{N} \subseteq \cX$. Requiring exact equality between $\hat{N}$ and $\Nt$, as in the classical $0/1$ loss, is often too restrictive: a prediction that captures most elements of $\Nt$ while including only a few irrelevant items should still be considered accurate.
  
  To capture this notion of approximate correctness, we adopt the standard measures of \textit{precision} and \textit{recall}. Precision is the fraction of predicted elements that belong to $\Nt$, while recall is the fraction of $\Nt$ that appears in the prediction. These measures are standard in problems where predictions are sets and quality depends on both relevance and coverage, including information retrieval~\citep{Manning08book}, data mining~\citep{han2011dataBook}, and other subset-prediction tasks. We remark that in the statistics literature, recall is also called \textit{sensitivity} or \textit{True Positive Rate (TPR)}, precision is also called \textit{Positive Predictive Value (PPV)}.

We study this problem in a \textit{sequential online learning setting}. At each round $t = 1, \dots, T$, the learner outputs a subset $N_t \subseteq\cX$. Then, one of two feedback types is selected uniformly at random: with probability $1/2$ the learner receives \textit{precision feedback}, and with probability $1/2$ it receives \textit{recall feedback}. In the precision case, a uniformly random element is drawn from the predicted set $N_t$, and the learner is given this element along with an indication of whether it belongs to $\Nt$ (receiving a reward of $1$ if it does). In the recall case, a uniformly random element is drawn from the target set $\Nt$, and the learner is given this element along with an indication of whether it belongs to $N_t$ (receiving a reward of $1$ if it does).

  Thus, each round provides information about a single element, sampled either from the prediction or from the target set. Precision feedback reflects the relevance of predicted items, while recall feedback reflects coverage of the target set. The learner must use this limited, sequential feedback to refine its predictions over time.
  
  The goal is to minimize \textit{regret}: we assume a hypothesis class $\cH \subseteq 2^{\cX}$ of candidate subsets is given, and define regret as the gap between the learner’s cumulative reward and that of the best fixed $N \in \cH$ in hindsight. We consider two standard settings. In the \textit{realizable} setting, there exists $N \in \cH$ that achieves perfect precision and recall with respect to $\Nt$. In the \textit{agnostic} setting, no such assumption is made, and the learner competes with the best $N \in \cH$.

The model can be viewed as an abstraction of repeated decision problems in which a system proposes a set of items and receives only item-level feedback.

For example, in content recommendation, a platform presents a list of items, and feedback may arise either from an item the user selects from the list (precision-type feedback) or from an item the user seeks out independently (recall-type feedback).

A related scenario arises in hiring. The goal of an employer or HR team is to identify strong candidates for a role, and they may proactively reach out to a set of candidates they believe are promising. Precision feedback arises when a candidate from this list is evaluated—for instance, a successful hire from the list provides positive feedback, while an unsuccessful hire provides negative feedback. At the same time, candidates may apply independently, outside the proposed list; when such a candidate turns out to be a successful hire, this provides recall-type feedback, indicating that a relevant element was missed.

Similar patterns appear in other domains where decisions are made over sets of candidates. For example, in cybersecurity, a detection system might output a set of IP addresses that seem suspicious. A security analyst may inspect a flagged IP address and determine whether it is malicious, yielding precision feedback. Alternatively, an attack may be observed from a malicious IP address that was not necessarily flagged, yielding recall feedback depending on whether that IP was included in the predicted set.

  These examples are intended mainly to build intuition. Our goal is not to faithfully model any specific application, but rather to isolate a basic theoretical learning problem: learning sets from item-level precision and recall feedback.

  \smallskip
  
  \textbf{Our Contributions.}
  Our main result is a complete characterization of learnability: {\it a hypothesis class $\cH$ is learnable in this setting if and only if it has finite VC dimension.}   
  This parallels the classical PAC characterization, where VC dimension governs learnability. At the same time, the learning problem here is structurally different: the feedback depends on the learner’s prediction and is obtained by sampling either from the predicted set or from the target set. This induces two complementary but asymmetric signals - precision (relevance) and recall (coverage) - which leads to algorithmic behavior that differs substantially from both supervised learning and classical bandit models.
  
  Concretely, we show:
  \begin{itemize}
  % [itemsep=2pt, parsep=2pt, topsep=0pt, partopsep=0pt,left=5pt]
      \item In the realizable setting, we propose an algorithm that selects the smallest hypothesis consistent with the observed recall feedback, and prove that it achieves regret $O(d \log^2 T)$, where $d$ is the VC dimension of $\cH$. We also show that any algorithm must incur regret at least $\Omega(d)$.
  
      \item In the agnostic setting, we provide an algorithm achieving regret $\tilde{O}(d^{1/4} T^{3/4})$, showing that sublinear regret is achievable even without realizability assumptions.
  
      \item We identify several phenomena that distinguish this setting from classical learning models. In the realizable case, selecting an arbitrary hypothesis in $\cH$ that is consistent with all feedback observed so far (as in empirical risk minimization) does not, in general, yield a sound learner. However, a regularized variant of ERM that prioritizes the \textit{smallest} consistent hypothesis in the class yields a vanishing mistake bound. In contrast, in the agnostic setting, proper learning may fail: achieving low regret sometimes requires using hypotheses outside the class $\cH$.
  
      We also establish a connection to online learning: any online classification algorithm with a finite mistake bound can be transformed into an algorithm for our setting with a comparable mistake bound. This connection is one-way: there exist hypothesis classes that are not online learnable, yet are learnable in our setting. This separation arises because learnability in our model is characterized by VC dimension, whereas online learnability is governed by the Littlestone dimension, which can be arbitrarily larger.
  
      Finally, in the realizable case, we show that recall feedback alone suffices for learning, and that precision feedback is not necessary. This phenomenon appears to be specific to the realizable setting; in particular, our algorithm for the agnostic case relies on both types of feedback.
  \end{itemize}
  
  Overall, even this basic set-learning model exhibits a range of non-trivial learning behaviors. Our results provide a complete characterization of learnability in this setting via VC dimension. A natural next step is to refine this understanding by determining optimal regret and mistake bounds for specific hypothesis classes. In particular, it would be interesting to understand whether the agnostic regret can be improved to the more familiar $\sqrt{dT}$ rate,

  and whether the $\log^2 T$ factor in the realizable case can be removed. We believe these questions present a number of interesting and challenging problems.

  \paragraph{Related Work.}
  Precision and recall are natural and standard metrics used broadly in machine learning, spanning applications from binary classification~\cite{juba2019precision}, multi-class classification~\cite{grandini2020metrics},
  regression~\cite{torgo2009precision}, and time series~\cite{DBLP:conf/nips/TatbulLZAG18} to information retrieval~\cite{arora2016evaluation} and generative models~\cite{sajjadi2018assessing}.
  Beyond precision-recall, another related metric---the area under the ROC curve (AUC)---has also been extensively studied in the history of binary classification~\cite{ DBLP:conf/nips/CortesM03, DBLP:conf/nips/CortesM04,DBLP:conf/icml/Rosset04, DBLP:journals/jmlr/AgarwalGHHR05},  with a focus on generalization. 
  
  Multi-label learning~\cite{McCallum1999Multi,Schapire2000BoosTexterAB} has been an area of study in machine learning, with various, primarily experimental approaches (see, e.g., \cite{Elisseeff01,Petterson11,NKapoor12} and~\cite{Zhang14survey,BOGATINOVSKI2022survey} for surveys). In multi-label learning, the training set consists of examples, each associated with multiple labels rather than just one. The goal is to train a model that can learn the relationships between the features of each example and all its labels. At test time, the learner predicts a list of labels for new examples, aiming to capture all the relevant labels, rather than just a single one. 
  
  Some works have studied Bayes-consistency of surrogate losses in multi-label learning, where direct loss minimization is often infeasible~\cite{pmlr-v19-gao11a}. Extensions include rank-based metrics like precision@$\kappa$ and recall@$\kappa$~\citep{Menon19}, which are loss functions defined under the constraint that the number of labels predicted by the model is limited to 
  $\kappa$. Recently,~\cite{mao2024multilabellearningstrongerconsistency} provided stronger $H$-consistency bounds with non-asymptotic guarantees. 

  Unlike multi-label learning, which trains on full per-example label sets, we consider an online learning problem and our goal is to identify a single hidden subset using only random positive samples and bandit-style precision feedback\footnote{Bandit-style means precision is only observable for the implemented hypothesis: given $N_t=\{A,B,C\}$, feedback comes from a sampled item in $N_t$
   and cannot evaluate other sets such as $\{D,E\}$. Thus, feedback is one-item and hypothesis-dependent.}, with no explicit negatives or complete label information. 

  A recent work \cite{Cohen24PR} considered precision and recall learning in the PAC model for user-dependent sets, and derived guarantees under recall‐only feedback; in contrast, our results deliver online regret bounds for precision and recall in a partial‐feedback setting.

  \vspace{-3mm}
  \section{Model}

In our model, we distinguish between a universe of items \(\cX\), on which the hypothesis class \(\cH \subseteq 2^{\cX}\) is defined, 
and a fixed available set \(X \subseteq \cX\) in a given learning instance. 
The learner \textit{knows} \(X\). 
For intuition, one may think of \(\cX\) as the set of all dishes a restaurant could potentially offer, 
while \(X\) is the current menu, determined by seasonality, ingredients, or the chef's daily choices. 
The learning task could then be to learn a personalized submenu for a customer based on the currently available dishes in \(X\). 
This distinction allows us to define learnability in a qualitative, ``distribution-free'' manner --- uniformly over all finite sets \(X \subseteq \cX\) --- even for infinite hypothesis classes over infinite domains \(\cX\). 
In particular, it clarifies the role of VC dimension as the quantity characterizing learnability. 
See \Cref{sec:Xdifference} for a discussion on \(X=\cX\) versus \(X\subseteq\cX\).

Let $\Nt\subseteq X$ be the \textit{unknown} target set. 

  We consider an online setting over $T$ rounds. At each round $t = 1, \ldots, T$, the learner selects a
  {predicted} set $N_t \subseteq X$, and one of the following two feedback types is selected with equal probability (i.e., probability $1/2$):

  \begin{itemize}
      \item \textbf{Recall feedback:} a uniformly random item $u_t \sim \Unif(\Nt)$ is sampled, 
      and the learner observes it and  receives recall reward $g_t(N_t) = \1(u_t \in N_t)$.
      \item \textbf{Precision feedback:} a uniformly random item $v_t \sim \Unif(N_t)$ is sampled, along with a label $y_t = \1(v_t \in \Nt)$, and the learner observes the item and receives precision reward $g_t(N_t) = y_t$.
  \end{itemize}

  The \textit{(expected) reward} of a set $N\subseteq X$ is defined as:
  \begin{align*}
  g(N) &= \frac{1}{2} \underbrace{\PPs{u \sim \Unif(\Nt)}{u \in N}}_{\text{recall}: r(N)} + \frac{1}{2} \underbrace{\PPs{v \sim \Unif(N)}{v \in \Nt}}_{\text{precision}: p(N)}\,.
  \end{align*}

  {The first term, recall $r(N)$, is the probability that a uniformly random target item $u \sim \mathrm{Unif}(\Nt)$ is included in $N$. The second term, precision $p(N)$, is the probability that a uniformly random item $v \sim \mathrm{Unif}(N)$ belongs to $\Nt$.}
  
  Recall can also be expressed as $r(N) = \frac{|N\cap \Nt|}{|\Nt|}$, which is the fraction of items in $\Nt$ that also belong to $N$. Similarly, precision can be expressed as $p(N) = \frac{|N\cap \Nt|}{|N|}$, which is the fraction of items in $N$ that also belong to $\Nt$.\footnote{We adopt the convention that $r(N) = 1$ when $|\Nt| = 0$ and $p(N) = 1$ when $|N| = 0$.}

  Given any universal set $N \subseteq \cX$ (which may include items {outside $X$}), 
  {the set effectively available to the learner is its restriction,} $N_{|X} = N \cap X$.
  For example, if $N = \Nt \cup N^{\text{unavailable}}$, where $\Nt$ is the target set and $N^{\text{unavailable}} \subseteq \cX \setminus X$ consists of some unavailable items, then none of the items in $N^{\text{unavailable}}$ 
  can be selected. Consequently, the set of items 
  $N$ coincides exactly with $\Nt$. In this case, $N$ achieves perfect precision and recall, and therefore receives the same reward as $\Nt$. We therefore generalize the definition of reward to any $N \subseteq \cX$: if $N$ contains items outside $X$, only the subset $N_{|X}$ can be played. Accordingly, we define the reward of any set $N$ as
  \[
  g(N) = g(N_{|X}).
  \]
  Then given a hypothesis class $\cH\subseteq 2^\cX$, the \textit{optimal} hypothesis is 
  \[
  N^\star = \argmax_{N \in \cH} g(N)\,.
  \]

  The learner’s goal is to minimize the regret:
  \[
  \text{Reg}_T = g(N^\star) \cdot T - \EE{\sum_{t=1}^T g_t(N_t)} = g(N^\star) \cdot T - \sum_{t=1}^T g(N_t)\,.
  \]
  We say a hypothesis class $\cH$ is \textit{learnable} if we can always achieve diminishing regret for any $X$ and $\Nt$. More formally, 
  \begin{definition}[Learnability] 
  A hypothesis class $\mathcal{H}$ is learnable if there exists a learning rule and a function $R(T, \delta)$ such that for any confidence parameter $\delta \in (0,1)$: (1) $R(T, \delta) = o(T)$ for any fixed $\delta$, and (2) for all available sets $X \subseteq \mathcal{X}$ and target sets $\Nt \subseteq X$, the learning rule (given $X$) achieves $\text{Reg}_T \leq R(T, \delta)$ with probability at least $1-\delta$.
  \end{definition}
  
  This raises the central question:
  \begin{center}
      \textbf{Which hypothesis classes are learnable?}
  \end{center}
  
  Note that our learning problem is not full-information. 
  While we can observe a recall reward for each hypothesis $N \in \cH$, i.e., $\1(u_t \in N)$ after receiving the recall feedback $u_t\sim \Unif(\Nt)$, the situation is different for precision. 
  In precision feedback, the item $v_t$ is sampled from the predicted list, so we can only observe a precision reward for the selected hypothesis, not for all hypotheses in $\cH$. \textit{Thus, recall feedback is full-information-style (i.e., evaluable across all hypotheses), whereas precision feedback is bandit-style (i.e., limited to the selected hypothesis).}
  Applying a standard bandit algorithm such as \textsc{EXP3} in this setting yields the regret bound:
  \[
  \text{Reg}_T = O\left(\sqrt{T |\cH| \log |\cH|}\right),
  \]
  which scales polynomially with the size of the hypothesis class $|\cH|$.

  Thus, any finite class is learnable. However, when the class is large, i.e., $|\cH| \geq T$, this regret bound becomes vacuous, and it remains unclear which infinite classes are learnable.

  In the remainder of the paper, we simplify the setting by assuming that in each round we observe both precision and recall feedback. This assumption is justified as follows: in the original setting, where only one type of feedback (precision or recall) is observed per round at random, we can simulate both types by repeatedly using the same hypothesis for $O(\log(T/\delta))$ rounds. With high probability (at least $1-\delta$), this ensures receiving both feedback types for each implemented hypothesis. As a result, any algorithm designed for the simplified setting with regret bound $R$ can be extended to the original setting with regret at most $O(\log(T/\delta)) \cdot R$. Any regret lower bound in this simplified setting is also a lower bound in the original setting.

  Each set $N \subseteq \cX$ naturally defines a binary classifier $h_N : \cX \to \{0,1\}$ via $h_N(x) = \1(x \in N)$. We define the VC dimension and Littlestone dimension of the hypothesis class $\cH$ as the VC and Littlestone dimensions of the corresponding class of binary functions $\{h_N \mid N \in \cH\}$. We now present regret bounds in terms of these dimensions.

  \section{Realizable Setting}\label{sec:realizable}
  In the realizable setting, we assume that the optimal set $N^\star$ incurs no loss, i.e., $g(N^\star) = 1$. We propose two algorithmic approaches: one based on the mistake-bound model from online learning, and the other inspired by a reduction from PAC learning to stochastic online learning, where in each round an Empirical Risk Minimizer is selected based on the historical data. The first approach yields a regret guarantee based on the Littlestone dimension, while the second yields a regret bound of $O(\vcd(\cH) \log^2 T)$ 
  based on the VC dimension. The former is quantitatively better when the hypothesis class has finite Littlestone dimension. We delegate all missing proofs to the appendix.

  \subsection{Reduction to Standard Online Learning}
  Any standard online learning algorithm $\cA$ for binary classification can be applied directly here (\Cref{alg:reduction}) and incurs the same amount of mistakes.
  \begin{algorithm}[t]\caption{Reduction to Online Classification}\label{alg:reduction}
    \begin{algorithmic}[1]
      \STATE \textbf{Input:} standard online learning algorithm $\cA$
      \FOR{$t = 1, \ldots, T$}
      \STATE Let $N_t = \{x|\cA \text{ predicts positive at } x\}$
      \IF{observe precision feedback $v_t$ and make a mistake $v_t \notin \Nt$} 
      \STATE Feed $\cA$ with $(v_t, 0)$.
          \ELSE

          \STATE  Feed $\cA$ with $(u_t, 1)$, where $u_t$ is a recall feedback.
      \ENDIF

      \ENDFOR
    \end{algorithmic}
  \end{algorithm}
  \begin{restatable}{theorem}{halving}\label{thm:halving}
  For any $X\subseteq \cX$ and $\Nt\subseteq X$, for any realizable hypothesis class $\cH$, for any online classification algorithm $\cA$ with mistake bound $M$ for $\cH$, \Cref{alg:reduction} will incur regret $\text{Reg}_T \leq M$. By choosing $\cA$ to be the standard optimal algorithm  by~\cite{littlestone1988learning}, \Cref{alg:reduction} will incur regret $\text{Reg}_T \leq \ld(\cH)$, where $\ld(\cH)$ is the Littlestone dimension of $\cH$.
  \end{restatable}

  Hence, any hypothesis class with finite Littlestone dimension is learnable in our setting. However, in the following, we show that the Littlestone dimension is not the appropriate complexity measure in our setting; instead, the VC dimension is the relevant one.

  \subsection{Hypothesis Classes with Finite VC Dimension are Learnable}
   We now present a different algorithm inspired by Empirical Risk Minimization (ERM). In standard 

   online/i.i.d. realizable setting, selecting any empirical risk minimizer—i.e., an arbitrary hypothesis that is consistent with the observed history—yields a regret bound of $O(\vcd(\cH) \log T)$. In our setting, we can also achieve a regret of 
   $O(\vcd(\cH) \log^2 T)$ using a similar idea: at round $t$, let $H_t = \{N\in \cH \mid u_\tau \in N, \forall \tau \in [t-1]\}$ denote the set of hypotheses consistent with all historical \textit{recall} feedback. Among them, we choose a hypothesis with a minimal size (intersecting with the available items $X$),
  \[
  N_t \in \argmin_{N  \in H_t} |N\cap X|\,,
  \]
  where $|N\cap X|$ is the number of available items predicted by $N$. This corresponds to selecting the hypothesis with maximum likelihood.
  
  \begin{restatable}{theorem}{relmle}\label{thm:mle}
      For any $X\subseteq \cX$ and $\Nt\subseteq X$, for any realizable hypothesis class $\cH$, choosing the maximum likelihood hypothesis in each round achieves regret of $(d\log(T) + \log(T/\delta)) \log T$ with probability at least $1-\delta$, where $d=\vcd(\cH)$.
  \end{restatable}
  
  Note that the maximum likelihood predictor is a specific hypothesis consistent with all historical recall feedback and does not utilize the precision feedback at all. A natural question raised here is: \textit{would any arbitrary hypothesis consistent with recall feedback work?} The answer is no. Consider an example of input space $\cX=\{0,1,\ldots,n\}$ and hypothesis class $\cH=\{\{0\}, \{0,1\},\ldots,\{0,n\}\}$. When $X = \cX $ and $\Nt = \{0\}$, all hypotheses in $\cH$ will be consistent with the recall feedback as we will only observe $u_t = 0$. But all the other hypotheses, e.g., $\{0,1\}$ will incur $\frac{1}{2}$ loss in precision. This means that, in general, ERMs do not work for this problem. Next, we show a lower bound that depends on the $\vcd(\cH)$.

  \begin{restatable}[Lower bound]{theorem}{lbvc}\label{thm:lb-vc} 
  For any $d>0$ and any hypothesis class $\mathcal H$ with $\vcd(\cH)=d$, there exists an available set $X$ such that for any (possibly randomized) online learner, there exists a target set $\Nt$ realizable by $\cH$ for which the learner will suffer expected regret $\EE{\text{Reg}_T}\geq d/50$ for $T\geq d/10$.
  \end{restatable}
  There is a gap of $\log^2 (T)$ between the upper and lower bounds. However, we show that it is impossible to improve the lower bound to $d\log^2 (T)$. The optimal regret bound is left as an open question.

  \begin{proposition}\label{prop:lbvc}
  For any $d>0$, there exists a hypothesis class $\mathcal{H}$ with $\text{VCdim}(\mathcal{H}) = \log|\mathcal{H}|=d$ such that for any $X\subseteq \cX$, we can achieve regret $\text{Reg}_T\leq d$.  
  \end{proposition}

  \section{Agnostic Setting}\label{sec:agnostic}
  In the agnostic setting, the reward of the optimal hypothesis, $g(N^\star) = \max_{N \in \cH} g(N)$, does not necessarily equal 1. 
  
  \subsection{Impossibility Result for Proper Algorithms}
  A learning algorithm is said to be \textit{proper} if, in each round, it chooses a hypothesis from the hypothesis class $\cH$. Most well-known online learning algorithms—such as Multiplicative Weights, Follow the Perturbed Leader, and \textsc{EXP3}—as well as our algorithms for the realizable setting, are proper. However, in the agnostic setting, proper learning rules do not work in general.
  
  Specifically, we establish the following impossibility result.
  \begin{restatable}{theorem}{thmbanditlb}
  \label{thm:banditlb}
      There exists a hypothesis class $\cH$ with $\vcd(\cH) = 1$ such that, for any $T$, there exists an available set $X$ and target set $\Nt\subseteq X$ for which any proper learning algorithm incurs expected regret $\EE{\text{Reg}_T} \geq C \cdot T$ for some universal numerical constant $C > 0$.
  \end{restatable}
  At a high level, in this construction, all hypotheses in $\cH$ are large and mutually disjoint. The target set selects $T$ hypotheses, and then includes a $\frac{1}{2} + \epsilon$ fraction of one of them, and a $\frac{1}{2}$ fraction from each of the remaining $T-1$ hypotheses. This effectively reduces the problem to a bandit instance with $T$ independent arms.
  
  However, readers might observe that the reward of every hypothesis in $\cH$ is quite low--roughly $\frac{1}{4}$--since the recall is approximately $\frac{1}{T}$ and the precision is around $\frac{1}{2}$. Notably, both the empty set $\emptyset$ and the full available set $X$ achieve a higher reward, namely $ g(X), g(\emptyset)\geq \frac{1}{2}$. It follows that whenever the best hypothesis in the class satisfies $g(N^\star) < \tfrac{1}{2}$, one can outperform it by simply outputting either $\emptyset$ or $X$. This observation motivates a natural improper learning rule: output the empty set $\emptyset$ or the full available set $X$ whenever it is detected that no hypothesis in $\cH$ performs well. Based on this improper algorithmic principle, we propose a new algorithm in the following.

  \subsection{The Algorithm}
  We define the optimality margin of $\cH$ as follows. 
  \begin{definition}[Optimality margin]
  Given a hypothesis class $\cH$ and available set $X$, let $N^\star \in \arg\max_{N\in\cH} g(N)$ be any maximizer of the expected reward. The \textit{optimality margin} of $\cH$ is defined as
  $$
    \alpha_\cH = 2g(N^\star) - 1,  
  $$
  which measures how much better the optimal hypothesis performs compared to the trivial hypotheses $\emptyset$ and $X$ ,which both achieve reward $1/2$ and therefore have margin $0$.
  \end{definition}
  Since $g(N^\star) = \frac{1}{2}(r(N^\star) + p(N^\star))$ and $g(N^\star) = \frac{1}{2} \alpha_\cH + \frac{1}{2}$, we have $r(N^\star) + p(N^\star) = 1 + \alpha_\cH$. Since both precision and recall are at most 1, the optimal hypothesis $N^\star$ has both $r(N^\star) \geq \alpha_\cH$ and $p(N^\star) \geq \alpha_\cH$.
  
  We construct a new algorithm, \textsc{APRIL} (\Cref{alg:subclass}), which outputs the empty set $\emptyset$ when it detects that no hypothesis in $\cH$ performs well. We show that this algorithm achieves regret $\widetilde{O}(d^{1/4} T^{3/4})$, where $d = \text{VCdim}(\cH)$, demonstrating that any hypothesis class with finite VC dimension remains learnable in the agnostic setting. Moreover, 
  when the optimality margin satisfies $\alpha_H\geq C$ for some contant $C>0$,
  the algorithm enjoys an improved guarantee of $O\left(\sqrt{T d \log \tfrac{T}{d}}\,\right)$. Importantly, the algorithm does not require prior knowledge of $\alpha_\cH$.

  As mentioned earlier, one of the main challenges in the algorithm design is that precision feedback is not full-information—i.e., we cannot observe precision feedback for every hypothesis in $\cH$, but only for the selected set. To address this issue, we propose an alternative approach for estimating precision by leveraging the property that $|\Nt| = \frac{p(N) \cdot |N\cap X|}{r(N)}$ for any $N$. We first estimate the size of the target set $|\Nt|$ using all historical precision and recall feedback, and then use this estimate to compute the precision of candidate sets. 
  
  Specifically, given a sequence of predicted set $N_1,\ldots,N_m \subseteq X$, an estimator for recall $\hat r(\cdot)$, and corresponding precision feedback $(y_t)_{t=1}^m$, where $y_t = \1(v_t \in \Nt)$ with $v_t \sim \Unif(N_t \cap X)$, we define the following estimator for $\nt$:
  \begin{align}
      \hat n^\star &=
      \text{EST}(\hat r(\cdot), (N_t, y_t)_{t=1})=
      \frac{1}{m} \sum_{t=1}^{m} \frac{\1(y_t = 1) \cdot |N_t\cap X|}{\hat r(N_t)} \,.\label{eq:esttarget}
  \end{align}
  However, this estimate is only reliable when the ratio $\frac{|N_t \cap X|}{|\Nt|}$, for all $N_t$ from which we have collected feedback, is both lower and upper bounded—i.e., when the sets have relatively high recall and precision.

  Based on this estimation method, our algorithm proceeds in three stages (\Cref{alg:subclass}): first, filter out all hypotheses with low recall; second, filter out all hypotheses with low precision; and third, refine the estimates for both recall and precision, and select the hypothesis with the highest estimated reward. 

  \begin{itemize}[parsep=0pt, topsep=0pt, partopsep=0pt,left=5pt]
      \item In the first stage, we choose an arbitrary set and use it for $T_1$ rounds to collect sufficient recall feedback. From these observations, we construct a recall estimator $\hat r_0$ that approximates $r(N)$ for all $N \in \cH$—i.e., $|\hat r_0(N) - r(N)|$ is small for all $N \in \cH$. We then filter $\cH_{|X}$ by retaining only those sets with high estimated recall, denoted $H_1$. From this filtered class $H_1$, we select one set $N_1 \in \argmax_{N \in H_1} \frac{\hat r_0(N)}{|N\cap X|}$. Intuitively, $N_1$ is expected to have good precision among the sets in $H_1$, since precision satisfies $p(N) = \frac{r(N) \cdot \nt}{|N\cap X|}$.
  
      \item In the second stage, we use $N_1$ for $T_2$ rounds and use the resulting feedback to compute an estimate $\hat n_0^\star$ of the target set size via \Cref{eq:esttarget}. We then refine the recall-filtered class $H_1$ by retaining only hypotheses with high estimated precision—evaluated using $\hat n_0^\star$—yielding a new candidate class $\hat\cH$.
  
      \item In the third stage, if $\hat\cH$ is empty, we predict $\emptyset$ for all remaining rounds. Otherwise, in each subsequent round, we select and predict the hypothesis in $\hat\cH$ with the highest estimated reward, while continuously updating our estimators for recall, precision, and the target set size.
  \end{itemize}

  \begin{algorithm}[t]\caption{$\textsc{APRIL (\textbf{A}daptive \textbf{P}recision‑\textbf{R}ecall }$ $\textsc{\textbf{I}nteractive \textbf{L}earner)}$}\label{alg:subclass}
    \begin{algorithmic}[1]
        \STATE \textbf{Input: }  margin parameter $\alpha$, confidence $\delta$.

        \STATE \textbf{Stage 1:} Pick an arbitrary set and select this set for $T_1=\lceil\frac{64}{\alpha^{2}}\big(d\log\frac{64d}{\alpha^{2}}
      +\log\!\frac{4}{\delta}\big)\rceil$ rounds.
        \STATE Let 
          \[\hat r_0(N) = \frac{1}{T_1}\sum_{t =1}^{T_1}\1(u_t\in  N)\,.\]
        \STATE Filter $H_1 = \{N\in \cH| \hat r_0(N)\geq \frac{7\alpha}{8}\}$ and pick $N_1 \in \argmax_{N\in H_1} \frac{\hat r_0(N)}{|N\cap X|}$.
        \STATE \textbf{Stage 2: }Select $N_1$ for $T_2 = \lceil\frac{64\log(1/\delta)}{\alpha^2}\rceil$ rounds and use these rounds to estimate $\nt$ by $$\hat n_0^\star = \text{EST}(\hat r_0, (N_1,y_1,N_1,y_2,\ldots,N_1,y_{T_2}))\,.$$
        \STATE Filter $\hat \cH = \{N\in H_1| \hat p_0(N):= \hat r_0(N) \cdot \frac{\hat n_0^\star}{|N\cap X|}\geq \frac{7\alpha}{16}\}$.
      \STATE \textbf{Stage 3: }
      \IF{$\hat{\cH} =\emptyset$}
        \STATE Select the empty set $\emptyset$ for the remaining rounds.
        \ELSE
       \FOR{$t=1,\ldots, T_3=T-T_1-T_2$}
          \STATE Update 
          $\hat r_t(N) = \frac{1}{t-1}\sum_{\tau =1}^{t-1}\1(u_\tau\in N)$, 
          \STATE $\hat n^\star_t = \text{EST}(\hat r_t, (N_\tau,y_\tau)_{1:t-1})$, and $\hat p_t(N) = \frac{\hat r_t(N) \cdot \hat n^\star_t}{|N\cap X|}\,.$
          \STATE Pick $N_t = \argmax_{N\in \hat \cH} \hat r_t(N) +\hat p_t(N)$.
          \ENDFOR
      \ENDIF
    \end{algorithmic}
  \end{algorithm}

  \begin{restatable}{theorem}{agnRegVC}\label{thm:agnRegVC}
      For any hypothesis class $\cH$ with $\vcd(\cH) = d$, any margin parameter $\alpha$, and any confidence parameter $\delta$, by running \Cref{alg:subclass} with input parameters $\alpha$ and $\delta$, with probability at least $1-\delta$, the regret satisfies
  \begin{align*}
  \text{Reg}_T &= O\bigg(\frac{d\log(d/\alpha^2)+\log(1/\delta)}{\alpha^2}  + \min\bigg(\frac{1}{\alpha}\sqrt{\big(d\log\tfrac{T}{d}+\log\tfrac{T}{\delta}\big)\,T}\,,\,\alpha\,T\bigg)\bigg).
   \end{align*}
  By setting
  $
  \alpha = \alpha^\star :=  ({d\log\!\tfrac{2eT}{d} + \log\tfrac{T}{\delta}})^{1/4}{T}^{-1/4},
  $
  we achieve
  \begin{align*}
  \text{Reg}_T &= O\bigg(\sqrt{\big(d\log\tfrac{T}{d}+\log\tfrac{T}{\delta}\big)\,T} + \big(d\log\tfrac{T}{d}+\log\tfrac{T}{\delta}\big)^{1/4}\,T^{3/4}\bigg).
  \end{align*}
  \end{restatable}

  \begin{remark}
      When $\alpha$ is chosen to be no greater than $\alpha_\cH$, $\hat \cH$ will not be empty and the regret is 
      \[
      O\left(\frac{1}{\alpha^{2}}\big(d\log\frac{d}{\alpha^{2}} + \log\frac{1}{\delta}\big) +\frac{1}{\alpha} \sqrt{T\left(d\log\big(\tfrac{T}{d}\big)+\log\big(\tfrac{T}{\delta}\big)\right)} \right).
      \]
      In this case, a larger value of $\alpha$ leads to smaller regret. Although we do not know $\alpha_\cH$ a priori, we can perform a binary search 
      over $[\alpha^\star,1]$ to find the largest feasible $\alpha$ 
      such that $\hat{\cH}$ is non-empty (terminating the search when the search range is within $\alpha^\star$). When the optimality margin $\alpha_\cH \geq \alpha^\star$, we achieve regret
      \begin{align*}    
      \text{Reg}_T &= O\bigg(\frac{\log T\left(d\log(d/\alpha_\cH^2)+\log(\log(1/\alpha_\cH)/\delta)\right)}{\alpha_\cH^2}   + \frac{1}{\alpha_\cH}\sqrt{T\big(d\log\tfrac{T}{d}+\log\tfrac{T\log(1/\alpha_\cH)}{\delta}\big)}\bigg),
      \end{align*}
      which simplifies to $O\left(\sqrt{T(d\log\tfrac{T}{d} + \log\tfrac{T}{\delta})}\right)$ when 
   $\alpha_\cH \geq C$ for some constant $C>0$.
      \end{remark}

      To prove Theorem~\ref{thm:agnRegVC}, we use Lemmas~\ref{lmm:conc-degree}–\ref{lmm:conc-precisionVC} to bound the estimation errors for the target set size, recall, and precision. By combining these bounds, we establish the overall regret guarantee.

  \begin{restatable}{lemma}{lmmConcDegree}
  \label{lmm:conc-degree}
   Given an estimation function $\hat r(\cdot)$, any sequence $N_1,\ldots,N_m$ and corresponding precision feedback $(y_t)_{t=1}^m$ where $y_t =\1(v_t\in \Nt)$ with $v_t\sim \Unif(N_t)$, if
      \begin{itemize}
          \item $|\hat r (N) - r(N)|\leq \epsilon, \forall N\in \{N_t|t\in [m]\}$;
          \item $\forall N\in \{N_t|t\in [m]\}, p(N)\geq \alpha, r(N)\geq \alpha$ with $\alpha\geq 2\epsilon$,
      \end{itemize}
      then with probability $1-\delta$, the target set size estimator defined in \Cref{eq:esttarget} satisfies
      \[\frac{|\text{EST}(\hat r, (N_t,y_t)_{t=1}^m) - \nt|}{\nt}\leq \frac{2}{\alpha} \sqrt{\frac{\log(1/\delta)}{m}} + \frac{2\epsilon}{\alpha}\,.\]
  \end{restatable}

  In the next lemma, we show that if the hypothesis class is good, i.e., $\alpha_\cH \geq \alpha$, then with high probability, the hypotheses in the filtered $\hat\cH$ have non-trivial precision and recall and that the optimal hypothesis $N^\star$ is in $\hat{\cH}$.
  \begin{restatable}{lemma}{lmmhatH}\label{lmm:hatH}
      For any hypothesis class $\cH$ with $\alpha_\cH \geq \alpha$ and any $X\subseteq \cX$, with probability at least $1-\delta$, we have that for all $N\in \hat \cH$, $\min(p(N),r(N))\geq \frac{3\alpha}{16}$ and $N^\star\in \hat \cH$.
  \end{restatable}

  In the following lemma, we show that in the third stage we obtain good estimations for the precision and recall for each $N\in\hat \cH$.
  
  \begin{restatable}{lemma}{lmmConcPrecisionVC}\label{lmm:conc-precisionVC}
  Given any hypothesis class $\mathcal H$ with VCdim($\mathcal H$) = $d$, a fixed $\delta\in(0,1)$ and an integer $T_{3}\ge2$, with probability at least $1-\delta$, for every $N\in\hat \cH$ and every $t\in\{2,\dots,T_{3}\}$ simultaneously,
  \[
  \big|\hat r_t(N)-r(N)\big|
  \le
  \sqrt{\frac{d\log\!\big(\tfrac{2e\,T_{3}}{d}\big)+\log\big(\tfrac{T_{3}}{\delta}\big)}{t-1}},
  \]
  and
  \begin{align*}
  \big|\hat p_t(N)-p(N)\big|
   \le &
  \big(1+\frac{64}{3\alpha}\big)
  \sqrt{\frac{d\log\big(\tfrac{2e\,T_{3}}{d}\big)+\log\big(\tfrac{T_{3}}{\delta}\big)}{t-1}}
   +
  \frac{64}{3\alpha}\cdot 
  \frac{d\log\big(\tfrac{2e\,T_{3}}{d}\big)+\log\big(\tfrac{T_{3}}{\delta}\big)}{t-1}.
  \end{align*}
  \end{restatable}
  By combining these lemmas, we prove \Cref{thm:agnRegVC}.

  \section{Discussion} 
  
  In this paper, we introduced a simple statistical learning framework for identifying an unknown target subset of a domain in the online setting. Under this framework, we showed that a hypothesis class is learnable if and only if it has finite VC dimension. However, a $\log^2 T$ factor gap remains between the upper and lower regret bounds in the realizable setting, which we leave as an open question. In the agnostic setting, we provided an upper bound on regret, but the optimal regret rate remains unknown. 
  
  Despite the simplicity of our setting, it has the potential to extend to more complex and practical scenarios. Below are some potential extensions.

\paragraph{Unknown available sets.} We assume that the finite available set $X\subseteq\mathcal X$ is known to the
learner. It remains open whether finite VC dimension continues to characterize
learnability when $X$ is fixed but not revealed, so that the learner observes
$X$ only through precision and recall feedback from the effective predictions
$N_t\cap X$.
  
  \paragraph{Different feedback types.} We considered a setting where both precision and recall feedback are available. Other types of feedback may also be worth exploring. For example, \cite{Cohen24PR} studies PAC learning with recall-only feedback, motivated by different application contexts.

  \paragraph{More general reward definitions.} Our definitions of precision and recall are based on the standard counting measure. These can potentially be generalized—for instance, by considering an underlying distribution over the domain and defining precision and recall accordingly. Furthermore, our reward function is defined as the average of precision and recall. Exploring alternative combinations or trade-offs between the two (e.g., weighted averages or F-measures) is a promising direction for future work.
  
  \paragraph{Multiple users.} 

  {This work focuses on learning a single target set. A natural and important extension is to a functional setting, in which the goal is to learn a set-valued function rather than a single set. This richer setting introduces new challenges, such as how to capture shared structure across target sets.}

  \section*{Acknowledgments}
Yishay Mansour has received funding from the European Research Council (ERC) under the European Union’s Horizon 2020 research and innovation program (grant agreement No. 882396), by the Israel Science Foundation,  the Yandex Initiative for Machine Learning at Tel Aviv University and a grant from the Tel Aviv University Center for AI and Data Science (TAD).

Shay Moran is a Robert J.\ Shillman Fellow; he acknowledges support by ISF grant 1225/20, by BSF grant 2018385, by Israel PBC-VATAT, by the Technion Center for Machine Learning and Intelligent Systems (MLIS), and by the European Union (ERC, GENERALIZATION, 101039692). Views and opinions expressed are, however, those of the author(s) only and do not necessarily reflect those of the European Union or the European Research Council Executive Agency. Neither the European Union nor the granting authority can be held responsible for them.

Lee Cohen is supported by the Simons Foundation Collaboration on the Theory of Algorithmic Fairness, and the Simons Foundation investigators award 17351.

\bibliographystyle{apalike}
\bibliography{ref}
  \newpage
  \appendix
  \onecolumn

  \section{Proofs for \Cref{sec:realizable}}
  \subsection{Proof of \Cref{thm:halving}}
  \halving*
  \begin{proof}[Proof of \Cref{thm:halving}]
  Since the setting is realizable, $h^* \in \cH$ never makes any mistake.  
  Whenever we make a mistake, there are two cases:
  \begin{itemize}
      \item \textbf{Precision mistake}: we observe an $v_t \sim \Unif(N_t)$ and $v_t\notin \Nt$. Then $\cA$ makes a mistake at $(v_t,0)$.
   
      \item \textbf{Recall mistake}: we observe an $u_t \sim \Unif(\Nt)$ and $u_t\notin N_t$. Then $\cA$ makes a mistake at $(u_t,1)$.
  \end{itemize}
  Hence, if $\cA$ has mistake bound $M$, we have regret $\text{Reg}_T\leq M$.
  \end{proof}
\subsection{Proof of \Cref{thm:mle}}
  \relmle*
  \begin{proof}[Proof of \Cref{thm:mle}]
      Let $\lr(N):=1-r(N)$ denote the recall loss.
      We reduce our setting to a standard PAC learning problem: for each $N \in \cH$, consider its corresponding binary classifier $h_N(x) = \1(x \in N)$ and we consider PAC learning of $\{h_N|N\in \cH\}$. In this PAC learning problem, let the data distribution be uniform over $N^{\text{target}}$. The classification error of $h_N$ under this distribution is 
      \begin{align*}
      \err(h_N) &:= \mathbb{E}_{x\sim \text{Unif}(N^{\text{target}})}[\1(h_N(x)\neq 1)] = \mathbb{E}_{x\sim \text{Unif}(N^{\text{target}})}[\1(x\notin N)]  = \lr(N)\,.
      \end{align*}
  
      Our algorithm can therefore be viewed as ERM in this binary classification problem. By the standard uniform convergence bound for classes with finite VC dimension, for any $t=1,\ldots, T-1$, with probability at least $1 - \delta/T$, any $h_N$ that is consistent with the $t$ observed samples (i.e., $N\in H_{t+1}$) satisfies 
      \[\lr(N) = \err(h_N)\leq O\left(\frac{d\log\!\big(\frac{t}{d}\big)+\log\big(\frac{T}{\delta}\big)}{t}\right)\,.\]

      For any $N \in H_{t+1}$, since $N^\star$ is consistent and must belong to $H_{t+1}$, $N^\star$ has perfect precision and recall, and $N_t$ is minimal, we have $|N_t\cap X|\leq |N^\star \cap X|$ and $\frac{|N^\star\cap X|}{\nt}=1$ 
      The precision of $N$ satisfies
      \begin{align*}
          p(N_t) &= r(N_t)\cdot \frac{\nt}{|N_t\cap X|} \\
          &\ge r(N_t)\cdot \frac{\nt}{|N^\star\cap X|} = r(N_t) \\
          &\ge 1- \frac{d\log\!\big(\frac{t}{d}\big)+\log\big(\frac{T}{\delta}\big)}{t} \,.
      \end{align*}
      By a union bound over all $t=1,\ldots,T$, we have that with probability at least $1-\delta$,
      \begin{align*}
      \text{Reg}_T &\leq \sum_{t=1}^T  \frac{d\log\!\big(\frac{t}{d}\big)+\log\big(\frac{T}{\delta}\big)}{t}  \leq \left( d\log\!\big(\frac{T}{d}\big)+\log\big(\frac{T}{\delta}\big)\right)\log(T).
      \end{align*}
  \end{proof}
  \subsection{Proof of \Cref{thm:lb-vc}}
  \lbvc*
  \begin{proof}[Proof of \Cref{thm:lb-vc}]
  Let $X=\{x_1,\dots,x_d\}$ be a shattered set for $\mathcal H$.  Draw $\Nt\subseteq X$ uniformly at random among all $\binom d{d/2}$ subsets of size $d/2$.  Fix any sequence of predictions $N_1,\dots,N_{d/10}$. Suppose $t\leq d/10$.
  
  If $|N_t|\ge 2d/3$, at least $N_t-t$ items of $N_t$ that have not been observed yet, and each is in $\Nt$ w.p. at most $1/2$.
  \begin{align*}
  \E\left[1-\text{precision}(N_t)\right]
  &=1-\frac{\E\left[|N_t\cap \Nt|\right]}{|N_t|}\\
  &\geq 1-\frac{1}{|N_t|}\left(t+\frac{|N_t|-t}{2}\right)\\
  &=1-\left(\frac{|N_t|/2}{|N_t|}+\frac{t}{2|N_t|}\right)\\
  &\geq \frac12-\frac{t}{2|N_t|}\\
  &\geq  \frac12- \frac{\frac{d}{10}}{2\cdot\frac{2d}{3}}\\
  &= \frac12- \frac{3}{40}>\frac{1}{3}.
  \end{align*}
  Where one before the last inequality follows since $|N_t|\ge 2d/3$ and $t\leq d/10$.

  Assume $|N_t|<2d/3$.  Since $|\Nt|=d/2$,
  \begin{align*}
  \E\left[1-\text{recall}(N_t)\right]
  &=1-\E\left[\frac{|N_t\cap \Nt|}{|\Nt|}\right]\\
  &=1-\frac{\,t+\tfrac{|N_t|-t}{2}\,}{\,d/2\,}\\
  &=\frac{d - |N_t| - t}{d}\\
  &\ge 1-\frac{\tfrac{2d}{3}+\tfrac{d}{10}}{d}\\
  &=\frac{7}{30}>\frac{1}{5}.
  \end{align*}

  In either case  
  \[
  \E\left[\max\{1-\text{precision}(N_t),\,1-\text{recall}(N_t)\}\right]\;\ge\;\frac15.
  \]
  Summing over $t=1,\dots,T$ gives
  \[
  \E[\text{Reg}_T]\;\ge\;\frac T5
  \ge\frac{\lfloor d/10\rfloor}5
  \ge\frac d{50}.
  \]
  Thus, there must exist some choice of $\Nt$ for which 
  $\text{Reg}_T\ge d/50$, establishing the claimed $\Omega(d)$ lower bound.
  \end{proof}
  \subsection{Proof of Proposition \ref{prop:lbvc}}
  \begin{proof}[Proof of Proposition \ref{prop:lbvc}]
  Let $\cX = \{x_1, \ldots, x_d\}$ and $\cH$ be all subsets of $\cX$. In each round $t \le d$, the algorithm predicts $N_t = \{x_t\}$. After $d$ rounds, we have learned for every $x_t \in X$ whether $x_t \in \Nt$, i.e., we learned $\Nt$. For future rounds, the algorithm just outputs $N_t = \Nt$.
  \end{proof}
  \section{Proofs for \Cref{sec:agnostic}}
  \subsection{Proof of \Cref{thm:agnRegVC}}
  \agnRegVC*

  \begin{proof}[Proof of \Cref{thm:agnRegVC}]
      According to Lemma~\ref{lmm:hatH}, if $\hat \cH = \emptyset$, it implies that $\alpha_\cH <\alpha$, then 
      \[\text{Reg}_T \leq O\left(
      \frac{1}{\alpha^{2}}\big(d\log\frac{d}{\alpha^{2}}
      +\log\frac{1}{\delta}\big)
      + \alpha T\right)\,.\]
      Otherwise, according to Lemma~\ref{lmm:conc-precisionVC}, we have the estimated reward $\hat g_t(N) = \frac{1}{2}(\hat r_t(N) +\hat p_t(N))$ of $N$ satisfy
      \[|\hat g_t(N) - g(N)| = O\left(\frac{1}{\alpha} \sqrt{\frac{d\log\big(\tfrac{T}{d}\big)+\log\big(\tfrac{T}{\delta}\big)}{t-1}} \right)\,,\]
      for all $t\in [T_3]$ and $N\in \hat \cH$.
      Then, since $N^\star\cap X$ is retained in $\hat H$ (see Lemma~\ref{lmm:hatH}), we have
      \begin{align*}
          g(N^\star)-g(N_t) &\leq \hat g_t(N^\star)-\hat g_t(N) + O\left(\frac{1}{\alpha} \sqrt{\frac{d\log\big(\tfrac{T}{d}\big)+\log\big(\tfrac{T}{\delta}\big)}{t-1}} \right)\\
          &\leq O\left(\frac{1}{\alpha} \sqrt{\frac{d\log\big(\tfrac{T}{d}\big)+\log\big(\tfrac{T}{\delta}\big)}{t-1}} \right).
      \end{align*}
      Hence,
      \begin{align*}
       \text{Reg}_T \leq &O\left(
      \frac{1}{\alpha^{2}}\big(d\log\frac{d}{\alpha^{2}} + \log\frac{1}{\delta}\big)\right) +
      O\left(\frac{1}{\alpha} \sqrt{T\left(d\log\big(\tfrac{T}{d}\big)+\log\big(\tfrac{T}{\delta}\right)\big)} \right),
      \end{align*} 
      which completes the proof.
  \end{proof}
  
  \subsection{Proof of Lemma~\ref{lmm:conc-degree}}
  \lmmConcDegree*
  
  \begin{proof}[Proof of Lemma~\ref{lmm:conc-degree}]
  
  For any set $N\subseteq X$, we have $\frac{p(N)}{r(N)}\cdot \frac{|N|}{\Nt}=1$ according to the definitions of precision and recall. Then we have
       \begin{align}
          &\frac{|\text{EST}(\hat r, (N_t,y_t)_{t=1}^m)-\Nt|}{\Nt} \nonumber\\
          = &\frac{1}{m}\left|\sum_{t =1}^{m}\left(\frac{\1(y_t =1)
          }{\hat r(N_t)} - \frac{p(N_t)}{r(N_t)}\right)\cdot \frac{|N_t|}{\Nt}\right|\nonumber\\
          \leq & \frac{1}{m}\left|\sum_{t =1}^{m}\left({\1(y_t =1)
          } - {p(N_t)
          }\right)\cdot \frac{|N_t|}{\Nt\cdot {r(N_t)}}\cdot \frac{r(N_t)}{{\hat r_t(N_t)}}\right| + \frac{1}{m}\left|\sum_{t =1}^{m}\left(\frac{p(N_t)
          }{\hat r_t(N_t)} - \frac{p(N_t)}{r(N_t)}\right)\cdot \frac{|N_t|}{\Nt}\right|\nonumber\\
          = & \frac{1}{m}\left|\sum_{t =1}^{m}\left({\1(y_t =1)
          } - {p(N_t)
          }\right)\cdot \frac{|N_t|}{\Nt\cdot {r(N_t)}}\cdot \frac{r(N_t)}{{\hat r_t(N_t)}}\right| + \frac{1}{m}\left|\sum_{t =1}^{m}\left(\frac{r(N_t)
          }{\hat r_t(N_t)} - 1\right)\right|\nonumber\\
          \leq & \frac{1}{m}\left|\sum_{t =1}^{m}\left({\1(y_t =1)
          } - {p(N_t)
          }\right)\cdot \frac{|N_t|}{\Nt\cdot {r(N_t)}}\cdot \frac{r(N_t)}{{\hat r_t(N_t)}}\right| + \frac{2\epsilon}{\alpha}\nonumber\\
          = & \frac{1}{m}\left|\sum_{t =1}^{m}\left({\1(y_t =1)
          } - {p(N_t)
          }\right)\cdot \frac{1}{p(N_t)}\cdot \frac{r(N_t)}{{\hat r_t(N_t)}}\right| + \frac{2\epsilon}{\alpha}\nonumber\\
          \leq & \frac{2}{\alpha} \sqrt{\frac{\log(1/\delta)}{m}} + \frac{2\epsilon}{\alpha} \label{eq:conc-degree}\,,
      \end{align}
      where the last inequality applies Azuma-Hoeffding inequality.
  \end{proof}

  \subsection{Proof of Lemma~\ref{lmm:hatH}}
  \lmmhatH*
  \begin{proof}[Proof of Lemma~\ref{lmm:hatH}]
      First, from uniform convergence, for each hypothesis $N\in \cH$ with probability at least $1-\delta$, 
      \begin{equation}\label{eq:rEst}
      |\hat r_0(N)-r(N)|\leq
      \frac{\alpha}{8}\,.
      \end{equation}

      Thus, with probability at least $1 - \delta$, $N^\star$ is in $H_1$ and for all $N\in H_1$ we have $r(N) \geq \frac{3\alpha}{4}$ and thus $\frac{|N \cap X|}{\Nt} = \frac{r(N)}{p(N)}\geq \frac{3\alpha}{4}$.
  
      Second, we want to show that $N_1$ has high precision so that we can obtain a good estimate of $\Nt$ by selecting $N_1$ for multiple rounds by Lemma~\ref{lmm:conc-degree}. 
      For any $N\in H_1$, we have 
      \[\hat r_0 (N) \leq^{(\ref{eq:rEst})} r(N) +\frac{\alpha}{8} \leq r(N) +\frac{1}{6}r(N) = \frac{7}{6}r(N)\,,\]
      and
      \[\hat r_0(N) \geq^{(\ref{eq:rEst})} r(N) - \frac{\alpha}{8} \geq r(N) - \frac{1}{7}\hat r_0(N)\,.\]
      From the above two equations, we can upper and lower bound the ratio between $\hat r_0(N)$ and $r(N)$, i.e.,
      \begin{equation}\label{eq:rhatBounds}
          \frac{7}{8} r(N) \leq \hat r_0(N)\leq \frac{7}{6}r(N)\,.
      \end{equation}

      Then we have 
      \begin{align*}
          \frac{\hat r_0 (N_1)}{|N_1\cap X|}\cdot \Nt \geq \frac{\hat r_0 (N^\star)}{|N^\star\cap X|}\cdot \Nt\geq \frac{7/8\cdot r (N^\star)}{|N^\star\cap X|}\cdot \Nt = \frac{7}{8} p(N^\star) \geq \frac{7}{8}\alpha\,,
      \end{align*}
      where the first inequality holds due to the definition of $N_1$.
      Also, we have
      \begin{align*}
          \frac{\hat r_0 (N_1)}{|N_1\cap X|}\cdot \Nt\leq \frac{7 r (N_1)}{6|N_1\cap X|}\cdot \Nt = \frac{7}{6} p(N_1)\,.
      \end{align*}
      Combining these two inequalities, we have $p(N_1)\geq \frac{3}{4}\alpha$.
      
      Then by applying Lemma~\ref{lmm:conc-degree} with $\epsilon = \frac{\alpha}{8}$ and $m=T_2$, we have
      \[\frac{|\hat n_0^\star-\nt|}{\nt} \leq \frac{1}{2}\,,\]
      which implies 
      \[\frac{\hat n_0^\star}{ \nt}\in [\frac{1}{2},2]\,.\]
      Then for any $N\in H_1$, if $\hat p_0(N) =\hat r_0(N) \cdot \frac{\hat n_0^\star}{|N\cap X|}\geq \frac{7\alpha}{16}$, then 
      \[p(N) = r(N) \cdot \frac{ \Nt}{|N\cap X|} \geq^{(\ref{eq:rhatBounds})} \frac{6}{7} \hat r_0(N) \cdot \frac{0.5 \hat n_0^\star}{|N\cap X|}\geq \frac{3\alpha}{16}\,.\]
      At the same time, 
      \[\hat p_0(N^\star) = \hat r_0(N^\star) \cdot \frac{\hat n_0^\star}{|N\cap X|} \geq  \frac{7}{8}r(N^\star) \cdot \frac{0.5  \Nt}{|N\cap X|} = \frac{7 p(N^\star)}{16} \geq \frac{7\alpha}{16}\,,\]
      which implies that $N^\star$ is in $\hat \cH$ w.p. at least $1-\delta$.
  \end{proof}

  \subsection{Proof of Lemma~\ref{lmm:conc-precisionVC}}
  \lmmConcPrecisionVC*
  \begin{proof}
  For any fixed $N\in\widehat H$ and $t$, Hoeffding’s inequality implies
  \[
  \Pr\{|\hat r_t(N)-r(N)|>\eps\}\le2\exp(-2(t-1)\eps^2).
  \]
  Let $\Pi_{\mathcal{H}}(\cdot)$ be the growth function of 
  $\mathcal H$. 
  A union bound over the $T_3-1$ rounds and an $\eps$-cover of size $\Pi_{\mathcal H}(2(t-1))$ implies that with probability at least $1-\delta$, for all $N\in\widehat H$ and $2\le t\le T_3$,
  \[
  |\hat r_t(N)-r(N)|\le\sqrt{\frac{\log(2\Pi_{\mathcal H}(2(t-1))\,T_3/\delta)}{2(t-1)}}.
  \]
  Since $\vcd(\mathcal H)=d$, Sauer-Shelah-Perels Lemma implies
  \[
  \Pi_{\mathcal H}(2(t-1))\le\sum_{i=0}^d\binom{2(t-1)}{i}<\big((2eT_3)/d\big)^d,
  \]
  so
  \[
  \log\big(2\Pi_{\mathcal H}(2(t-1))\,T_3/\delta\big)\le d\log\big((2eT_3)/d\big)+\log(T_3/\delta).
  \]
  Hence
  \begin{equation}\label{eq:recEst}
  |\hat r_t(N)-r(N)|\le\sqrt{\frac{d\log\big((2eT_3)/d\big)+\log(T_3/\delta)}{t-1}}.
  \end{equation}
  By Lemma 1 with $m=t-1$ and $p(N)\ge3\alpha/16$, and a union bound over $t$,
  with probability $\ge1-\tfrac\delta2$,
  \[
  \forall\,t:\quad
  \left|\frac{\hat n^\star_t}{ \nt}-1\right|
  +\tfrac{2\,\varepsilon}{\alpha}
  \le
  \tfrac{64}{3\alpha}
  \sqrt{\frac{d\ln\big(\frac{2eT_3}{d}\big)+\ln\big(\frac{T_3}{\delta}\big)}{\,t-1\,}},
  \]
  and hence
  \begin{equation}\label{eq:nEst}
  \big|\hat n^\star_t -  \nt\big|
  \le
  \tfrac{64}{3\alpha}
  \sqrt{\frac{d\ln\big(\frac{2eT_3}{d}\big)+\ln\big(\frac{T_3}{\delta}\big)}{\,t-1\,}}\; \nt
  \end{equation}
  For any $N\in\widehat H$,
  \[
  \hat p_t(N)
  =\frac{\hat r_t(N)\,\hat n^\star_t}{|N\cap X|},
  \quad
  p(N)
  =\frac{r(N)\, \nt}{|N\cap X|}.
  \]
  Thus
  \[
  \big|\hat p_t(N)-p(N)\big|
  \le
  \frac{\hat n^\star_t}{|N\cap X|}\,\big|\hat r_t(N)-r(N)\big|
  +\frac{r(N)}{|N\cap X|}\,\big|\hat n^\star_t -  \Nt\big|.
  \]
  Since 
  $\tfrac{\hat n^\star_t}{|N\cap X|}\le2/r(N)\le\tfrac{32}{3\alpha}$
  and
  $\tfrac{r(N)}{|N\cap X|}=p(N)/ \nt\le\tfrac{1}{\nt}$,
  combining \Cref{eq:recEst} and \Cref{eq:nEst} yields the bound.
  \end{proof}

  \subsection{Proof of \Cref{thm:banditlb}}
  \thmbanditlb*
  \begin{proof}[Proof of \Cref{thm:banditlb}]
      We will construct an example and reduce the problem to a bandit problem.
      Consider a set of $k$ non-overlapping hypotheses $\cH = \{N_1,\ldots,N_k\}$. 
      The size $|N|$ of each hypothesis $N$ is identical and is very huge, i.e., $|N| =n \gg T$ for all $N\in \cH$. Let available set $X=\cX=\cup_{N\in \cH} N$. 
  
      Now we construct two worlds with different random generation ways of the target set.
  
      In world I, the target set $\Nt$ is created randomly by selecting $(\frac{1}{2}+\epsilon) n$ items from the hypothesis $N_1$ and $(\frac{1}{2}-\frac{\epsilon}{k-1}) n$ from each of the remaining $k-1$ hypotheses, $N_2,\ldots,N_k$. Then no matter what $\Nt$ is generated, $\nt = \frac{kn}{2}$ and the recalls and precisions of hypotheses are
      \begin{align*}
          &r(N_1) = \frac{1+2\epsilon}{k}\,,p(N_1) = 1/2+\epsilon,\\
          &r(N_i) = \frac{1-\frac{2\epsilon}{k-1}}{k }\,,p(N_i) = 1/2-\frac{\epsilon}{k-1}, \forall i>1.
      \end{align*}
      For any fixed policy $\pi$, let $\cP_\textrm{I}$ denote the distribution over the sequence of decisions and feedbacks $((N_1,u_1,v_1,y_1),\ldots,(N_T,u_T,v_T,y_T))$ in world I.

      Under this fixed policy $\pi$, let $T_i$ denote the number of rounds we select hypothesis $N_i$. 
      Pick $i' = \argmin_{i>1} \EEs{\cP_\textrm{I}}{T_i}$. 
      Hence $\EEs{\cP_\textrm{I}}{T_i}\leq \frac{T}{k-1}$. 
      Then we construct world II as follows.
      In world II, the target set $\Nt$ is created randomly by selecting $(\frac{1}{2}+\epsilon) n$ items from the hypothesis $N_1$, selecting $(\frac{1}{2}+2\epsilon) n$ items from $N_{i'}$, and $(\frac{1}{2}-\frac{3\epsilon}{k-2}) n$ items from each of the remaining hypotheses. 
      Then no matter what $\Nt$ is generated, $\nt = \frac{kn}{2}$ and the recalls and precisions of hypotheses are
      \begin{align*}
          &r(N_1) = \frac{1+2\epsilon}{k}\,,p(N_1) = 1/2+\epsilon, \\
          &r(N_{i'}) = \frac{1+4\epsilon}{k}\,,p(N_{i'}) = 1/2+2\epsilon,\\
          &r(N_i) = \frac{1-\frac{6\epsilon}{k-2}}{k}\,,p(N_i) = 1/2-\frac{3\epsilon}{k-2}, \forall i\neq 1,i'.
      \end{align*}
      
      Let $\cP_\textrm{II}$ denote the distribution over the sequence of decisions and feedbacks over $T$ steps in world II, $((N_1,u_1,v_1,y_1),\ldots,(N_T,u_T,v_T,y_T))$.
      Following the standard proof techniques in minimax lower bounds for bandit problems, we want to show that the policy will incur high regret in at least one of the world. To disentangle with the choice of $\Nt$, instead of considering $\cP_\textrm{I}$ and $\cP_\textrm{II}$ themselves directly, we consider their alternatives. 
  
      At each round $t$, let $i_t$ denote index of the hypothesis $N_i$ in which the recall feedback $u_t\sim \Unif(\Nt)$ falls into, i.e., $u_t\in N_{i_t}$. Notice that the probability of $i_t =i$ is exactly $r(N_i)$. The generation process of $\cP_\textrm{I}$ is equivalent to:
      \begin{itemize}
          \item Sample $\Nt$ by randomly picking $r(N_i)\cdot n$ items from $N_i$ for all $i=1,\ldots,k$.
          \item Sample $i_t$ from  a categorical distribution $\cC = (r(N_1),r(N_2),\ldots,r(N_k))$.
          \item Sample $u_t\sim \Unif(N_{i_t}\cap \Nt)$.
          \item For any $N_t=N$ chosen by policy $\pi$ at round $t$, sample $v_t\sim \Unif(N)$ and $y_t = \1(v_t\in \Nt)$.
      \end{itemize}

      Now we construct an alternative $\cP_\textrm{I}'$ for $\cP_\textrm{I}$:
      \begin{itemize}
          \item Sample $i_t$ from the same categorical distribution $\cC = (r(N_1),r(N_2),\ldots,r(N_k))$.
          \item Sample $u_t\sim \Unif(N_{i_t})$.
          \item For any $N_t=N$ chosen by policy $\pi$ at round $t$, again $v_t\sim \Unif(N)$ is still sampled uniformly at random from $N$. But we choose $y_t\sim \Ber(p(N))$.
      \end{itemize}
      We apply the same procedure of generating an alternative $\cP_\textrm{II}'$ for $\cP_\textrm{II}$.
      We prove that $\cP$ and $\cP'$ are close enough, i.e., $\textrm{TV}(\cP_\textrm{I}',\cP_\textrm{I})\rightarrow 0$ and $\textrm{TV}(\cP_\textrm{II}',\cP_\textrm{II})\rightarrow 0$ when $n\rightarrow \infty$ by \Cref{lmm:alternative}.

      Now we are ready to prove the lower bound based on the standard proof techniques for bandit minimax lower bounds, see Chapter 15.2 of \cite{Lattimore-Szepesvari-Book} for example.
          In world I, no matter what $\Nt$ is generated, the best hypothesis is $N_1$ and the regret is at least $\frac{\epsilon T}{4}$ if $T_1\leq \frac{T}{2}$. However, we will suffer regret of $\frac{\epsilon T}{4}$ if  $T_1>\frac{T}{2}$ in world II. Let $\text{Reg}_T(\cP_\textrm{I})$ and $\text{Reg}_T(\cP_\textrm{II})$ denote expected regret in worlds I and II respectively. Then we have
          \begin{align*}
              &\text{Reg}_T(\cP_\textrm{I}) +\text{Reg}_T(\cP_\textrm{II})  \\
              >&\PPs{\cP_\textrm{I}}{T_1\leq \frac{T}{2}}\frac{\epsilon T}{4}+  \PPs{\cP_\textrm{II}}{T_1> \frac{T}{2}}\frac{\epsilon T}{4}\\
              \geq &\PPs{\cP_\textrm{I}'}{T_1\leq \frac{T}{2}}\frac{\epsilon T}{4}+  \PPs{\cP_\textrm{II}'}{T_1> \frac{T}{2}}\frac{\epsilon T}{4}-\textrm{TV}(\cP_\textrm{I}',\cP_\textrm{I}) -\textrm{TV}(\cP_\textrm{II}',\cP_\textrm{II})\\
              \geq& \frac{\epsilon T}{8}\exp(-\KL(\cP_\textrm{I}'\|\cP_\textrm{II}'))\,,
          \end{align*}
  where the last inequality applies the Bretagnolle-Huber inequality (Theorem 14.2 of \cite{Lattimore-Szepesvari-Book}) and \Cref{lmm:alternative}. It remains to upper bound $\KL(\cP_\textrm{I}'\|\cP_\textrm{II}')$. By chain rule of KL divergence and Lemma 15.1 of \cite{Lattimore-Szepesvari-Book}, we have
      \begin{align*}
          \KL(\cP_\textrm{I}'\|\cP_\textrm{II}') = \KL(\cP_\textrm{I}'(i_{1:T}) \|\cP_\textrm{II}'(i_{1:T})) + \sum_i \EEs{\cP_\textrm{I}'}{T_i}\KL(\cP_\textrm{I}'(v,y|N_i) \| \cP_\textrm{II}'(v,y|N_i))\,,
      \end{align*}
      where $\cP(i_{1:T})$ denotes the distribution of sequence $i_{1:T}$ and $\cP(v,y|N_i)$ denotes the distribution of precision feedback $(v,y)$ when selecting $N_i$ under generation process $\cP$. 
  
      By Lemma~\ref{lmm:kl}, we have
      \[\KL(\cP_\textrm{I}'(i_{1:T}) \|\cP_\textrm{II}'(i_{1:T})) \leq T\cdot \frac{8\epsilon^2}{k}\,.\]

      As aforementioned, the distribution over $v$ is identical in both worlds, but $y$ is sampled from $\Ber(p(N_i))$. For any two Bernoulli distributions parametrized by $q,q'$, we have
      \[\KL(\Ber(q)\|\Ber(q'))\leq \frac{(q-q')^2}{q'(1-q')}\,.\]
      Hence, we have 
      \[\sum_i \EEs{I}{T_i}\KL(\cP_\textrm{I}'(v,y|N_i) \| \cP_\textrm{II}'(v,y|N_i)) \lesssim T \cdot \frac{16\epsilon^2}{k^2} + \frac{T}{k-1} \cdot 16\epsilon^2\,.\] 
      Hence, we have
      \[\KL(\cP_\textrm{I}'\|\cP_\textrm{II}') \leq T \cdot \frac{40\epsilon^2}{k}\]
  Then by setting $\epsilon = \min\{\sqrt{\frac{k}{80 T}},1\}$, we have $\KL(\cP_\textrm{I}'\|\cP_\textrm{II}') \leq \frac{1}{2}$ when $k\leq T$. Thus,
  \[\text{Reg}_T(\cP_\textrm{I}) +\text{Reg}_T(\cP_\textrm{II}) \geq \frac{\epsilon T}{8} \exp(-\frac{1}{2}) \geq C\sqrt{kT},\]
  for a universal constant $C>0$. By setting $k= T$,
  we are done with the proof.
  \end{proof}

  \begin{restatable}{lemma}{alternative}\label{lmm:alternative}
      $\textrm{TV}(\cP_\textrm{I}',\cP_\textrm{I})\rightarrow 0$ and $\textrm{TV}(\cP_\textrm{II}',\cP_\textrm{II})\rightarrow 0$ when $n\rightarrow \infty$.
  \end{restatable}
  \begin{proof}[Proof of \Cref{lmm:alternative}]
      When $n\rightarrow\infty$, it is almost sure that there is no reappearance in $(u_1,u_2,\ldots,u_T)$. Conditional on no reappearance, every sequence will appear with equal probability under both $\cP_\textrm{I}$ and $\cP'_\textrm{I}$.
      Hence $\textrm{TV}(\cP'_\textrm{I}(u_{1:T}),\cP_\textrm{I}(u_{1:T}))\rightarrow 0$.
  
      For any $t = 1,\ldots, T$, let $F_{t} = ((N_1,u_1,v_1,y_1),\ldots,(N_{t-1},u_{t-1},v_{t-1},y_{t-1}), N_t)$ denote the history up to round $t$. In round $t$, no matter what $N_t =N$ is chosen, the distribution of $v_t\sim \Unif(N)$ is identical in $\cP'_\textrm{I}$ and $\cP_\textrm{I}$. Since $|N|=n\rightarrow \infty$ is large for any $N\in \cH$, it is almost sure that there is no reappearance in $(v_1,v_2,\ldots,v_T)$. Conditional on no reappearance, 
      \[\cP_\textrm{I}(y_t=1|F_{t})= \EE{\PPc{v_t\in \Nt}{\Nt,F_t}}=\EE{\frac{|N\cap \Nt|- t+1}{n}} \rightarrow p(N)\] 
      as $n\rightarrow \infty$.
  \end{proof}

  \begin{lemma}\label{lmm:kl}
      We have $\frac{1}{T}\KL(\cP_\textrm{I}'(i_{1:T}) \|\cP_\textrm{II}'(i_{1:T}))\leq \frac{20\epsilon^2}{k}$ for all $\epsilon\leq \frac{1}{8}$ and $k\geq 11$.
  \end{lemma}
  \begin{proof}[Proof of \Cref{lmm:kl}]
  In world $\cP_\textrm{I}'$, $i_{1:T}$ are i.i.d samples from $\cP_\textrm{I}'(i)$ with $\PP{i_t = 1} = \frac{1 + 2\epsilon}{k}$ and $\PP{i_t = i} = \frac{1 - \frac{2\epsilon}{k-2}}{k},\forall i\notin \{1,i'\}$. In world II, $i_{1:T}$ are i.i.d samples from $\cP_\textrm{II}'(i)$, with $\PP{i_t = 1} = \frac{1 + 2\epsilon}{k}$, $\PP{i_t = i'} = \frac{1 + 4\epsilon}{k}$, and $\PP{i_t = i} = \frac{1 - \frac{6\epsilon}{k-2}}{k},\forall i\notin \{1,i'\}$.

  We compute the KL divergence from $\cP_\textrm{I}'(i)$ to $\cP_\textrm{II}'(i)$:
  \begin{align*}
      D_{\text{KL}}(\cP_\textrm{I}'(i) \| \cP_\textrm{II}'(i)) 
  =& \frac{1 + a}{k} \log \left( \frac{1 + a}{1 + 4\epsilon} \right)
  + (k - 2) \cdot \frac{1 + a}{k} \log \left( \frac{1 + a}{1 + c} \right)\\
  =& \frac{1 + a}{k} \left((k-1)\log(1+a) - \log(1+4\epsilon) -(k-2)\log(1+c)\right)\,,
  \end{align*}
  where $a := -\dfrac{2\epsilon}{k - 1}$ and $c := -\dfrac{6\epsilon}{k - 2}$.
  
  For any $x\in [-\frac{1}{2},1]$, we have
  \[\ln(1+x) \leq x-\frac{x^2}{2}+\frac{x^3}{3}\,,\]
  and
  \[\ln(1+x) \geq x-x^2\,.\]
  Then we have
  
  \begin{align*}
      D_{\text{KL}}(\cP_\textrm{I}'(i) \| \cP_\textrm{II}'(i)) 
  =&\frac{1 + a}{k} \left((k-1)(a - \frac{1}{2}a^2 + \frac{a^3}{3}) - (4\epsilon - 16\epsilon^2) -(k-2) (c - c^2)\right)\\
  \leq &\frac{1 + a}{k} \left( 16\epsilon^2 +(k-2) c^2\right)\\
  \leq &\frac{20\epsilon^2}{k} \,,
  \end{align*}
  when $k\geq 11$.

  \end{proof}
  
  \section{$X\subseteq \cX$ vs $X= \cX$}\label{sec:Xdifference}
  All results in this paper still hold when fixing $X= \cX$ except \Cref{thm:lb-vc}.
  
  To clarify why the distinction matters in that case, consider the following example in the realizable setting:
  \begin{itemize}
  \item Let the domain $\cX$ be the disjoint union of $D \cup E_1 \cup \dots \cup E_{2^d}$, where $|D| = d$ and each $|E_i|$ is much larger than $d$.
  \item There is a one-to-one mapping between each subset $C \subseteq D$ (there are $2^d$ of them) and a corresponding $E_i$.
  \item Define the hypothesis class:
  \[
  \cH = \{ C \cup E_i : C \subseteq D \text{ and } E_i \text{ corresponds to } C \}.
  \]
  This class has a VC dimension $d$.
  \end{itemize}
  Now consider the realizable setting with $X = \cX$ and some $C \cup E_i$ as the ground-truth target. Since each $E_i$ is much larger than $d$, a recall example is likely to hit $E_i$ in the first round, revealing the correct hypothesis with high probability. Thus, \textit{the total loss is negligible despite the high VC dimension}, meaning VC dimension fails to lower bound regret in this case.
  On the other hand, if we restrict $X = D$, then the recall examples are forced to come from the uninformative core $D$, and the total loss becomes \textit{lower bounded by the VC dimension}, as expected.

\end{document}

%% file: macros.tex
\DeclareMathOperator*{\argmax}{arg\,max}
\DeclareMathOperator*{\argmin}{arg\,min}

\newcommand{\Ber}{\mathrm{Ber}}
\newcommand{\Unif}{\mathrm{Unif}}

\newcommand{\1}{\mathds{1}}

\newcommand{\E}{\mathbb{E}}
\newcommand{\EE}[1]{\mathbb{E}\left[#1\right]}

\newcommand{\EEs}[2]{\mathbb{E}_{#1}\left[#2\right]}

\newcommand{\PP}[1]{\mathbb{P}\left(#1\right)}
\newcommand{\PPs}[2]{\mathbb{P}_{#1}\left(#2\right)}

\newcommand{\PPc}[2]{\mathbb{P}\left[#1\left|#2\right.\right]}

\newcommand{\cA}{\mathcal{A}}

\newcommand{\cC}{\mathcal{C}}

\newcommand{\cH}{\mathcal{H}}

\newcommand{\cP}{\mathcal{P}}

\newcommand{\cX}{\mathcal{X}}

\newcommand{\eps}{\varepsilon}
\renewcommand{\epsilon}{\varepsilon}
\renewcommand{\hat}{\widehat}
\renewcommand{\tilde}{\widetilde}

\newcommand{\nothere}[1]{}

\newcommand{\err}{\mathrm{err}}
\newcommand{\vcd}{\mathrm{VCdim}}
\newcommand{\ld}{\mathrm{Ldim}}

\newcommand{\lr}{\ell^{\textrm{recall}}}

\newcommand{\Nt}{N^\text{target}}
\newcommand{\nt}{|N^\text{target}|}
\newcommand{\KL}{D_{\mathrm{KL}}}